%% file: main.tex
\documentclass{article}
\usepackage{natbib}  
\usepackage{graphicx} 
\usepackage{xcyauxlib}
\usepackage[margin=2cm]{geometry}  
\usepackage{booktabs}

\title{Combating Spurious Correlations in Graph Interpretability via Self-Reflection}

\author{Kecheng Cai, Chenyang Xu, Chao Peng, Jiafu Huang, Qiyuan Liang, Irene Zheng}
\date{January 2026}
\begin{document}

\maketitle

\begin{abstract}
  \input{05-abstract}
\end{abstract}

\input{10-intro}

\input{20-pre}

\input{30-frame}
\input{35-tun}

\input{60-con}
\bibliographystyle{plainnat}
\bibliography{sample}
\input{99-app}

\end{document}

%% file: 05-abstract.tex
Interpretable graph learning has recently emerged as a popular research topic in machine learning. The goal is to identify the important nodes and edges of an input graph that are crucial for performing a specific graph reasoning task. A number of studies have been conducted in this area, and various benchmark datasets have been proposed to facilitate evaluation. Among them, one of the most challenging is the Spurious-Motif benchmark, introduced at ICLR 2022. The datasets in this synthetic benchmark are deliberately designed to include spurious correlations, making it particularly difficult for models to distinguish truly relevant structures from misleading patterns. As a result, existing methods exhibit significantly worse performance on this benchmark compared to others.

In this paper, we focus on improving interpretability on the challenging Spurious-Motif datasets. We demonstrate that the self-reflection technique, commonly used in large language models to tackle complex tasks, can also be effectively adapted to enhance interpretability in datasets with strong spurious correlations. Specifically, we propose a self-reflection framework that can be integrated with existing interpretable graph learning methods. When such a method produces importance scores for each node and edge, our framework feeds these predictions back into the original method to perform a second round of evaluation. This iterative process mirrors how large language models employ self-reflective prompting to reassess their previous outputs. We further analyze the reasons behind this improvement from the perspective of graph representation learning, which motivates us to propose a fine-tuning training method based on this feedback mechanism. This method leads to further performance improvements. Experimental results show that our framework significantly improves existing methods, not only on Spurious-Motif but also on other popular graph interpretability benchmarks.

%% file: 10-intro.tex
\section{Introduction}



Interpretable graph learning is an emerging area in machine learning, aiming to make graph neural networks not only accurate but also understandable. While traditional graph learning methods often function as black boxes, studying their interpretability helps us figure out the relationships and patterns captured by these models. This is particularly important in high-stakes applications like drug discovery~\cite{DBLP:journals/jcc/ZengFLLWL24}, social network analysis~\cite{DBLP:journals/isci/KumarMKP22}, and fraud detection~\cite{DBLP:journals/tbd/LiYLW23}. 
The main task of interpretable graph learning is to identify important nodes, edges, or substructures. As many real-world complex tasks involve graph-structured data, identifying these elements can enhance transparency, foster trust in AI models, and facilitate debugging and validation processes.

There has been a growing body of work on interpretable graph learning in recent years~\cite{magister2022encoding,zhang2022protgnn,ragno2022prototype,miao2022interpretable,chen2022learning,serra2024l2xgnn}. Despite these advances, many existing methods still face a critical challenge: they are highly susceptible to \emph{spurious correlations}. In both real-world and synthetic scenarios, models may attend to graph components that are statistically correlated with the target label but are not causally responsible for the prediction. This problem is particularly pronounced in benchmarks such as Spurious-Motif~\cite{wu2022discovering}, where distractor subgraphs are deliberately introduced during training. As a result, explanation methods often highlight irrelevant structures, thereby undermining their reliability and generalization. These highlighted components, though predictive during training, are typically associated with spurious correlations and do not reflect the true decision logic of the model.




In large language models (LLMs), hallucination—the generation of factually incorrect or logically inconsistent content—has been widely documented~\cite{maynez-etal-2020-faithfulness,survey-of-hallucination}. An analogous failure mode arises in interpretable graph learning: when an explainer highlights spurious substructures that do not support the model’s prediction, the resulting rationale diverges from ground truth. We term this an explanation-level hallucination, emphasizing that the issue lies not only in predictive accuracy but in the faithfulness of the explanation itself.

A promising remedy in NLP is self-reflection, where a model explicitly reasons, critiques, and revises before finalizing its output. A prominent instantiation is Chain-of-Thought (CoT) and its variants, which encourage step-by-step reasoning prior to the final answer and have been shown to improve robustness and reduce hallucinations by leveraging internal knowledge and verification~\cite{DBLP:journals/corr/abs-2303-11366,DBLP:conf/acl/DhuliawalaKXRLC24,DBLP:journals/corr/abs-2310-06271,DBLP:conf/nips/KojimaGRMI22,DBLP:journals/corr/abs-2310-03951,DBLP:journals/corr/abs-2307-05300}. Motivated by these advances, we ask: can self-reflection likewise mitigate explanation-level hallucination in graph explanation? This question motivates our approach.

\subsection{Our Contributions}

This paper proposes a novel approach to improving interpretability in graph neural networks through self-reflection. Our main contributions are summarized as follows:
\begin{itemize}
    \item We introduce a lightweight, training-free framework that iteratively refines edge-level importance scores by repeatedly feeding the output of an interpretable model back into itself, allowing it to perform a form of self-reflection. This mechanism promotes internal consistency and progressively filters out spurious edges. Notably, our framework operates without modifying the original model architecture, making it easily compatible with a wide range of existing interpretable methods.

    \item  We provide a formal optimization perspective of the proposed self-reflection framework and prove that optimal solutions exhibit consistency across iterations (\cref{thm:consistent}). Based on this theoretical insight, we further propose a fine-tuning strategy specifically tailored to the framework, enhancing its effectiveness.

    \item We conduct experiments on the Spurious-Motif benchmark and other datasets, showing that our self-reflection framework, which repeatedly applies an interpretable model to refine its own explanations, can consistently improve classification accuracy, especially under strong spurious correlations. In addition, we demonstrate that a simple fine-tuning strategy further improves performance, particularly in terms of AUC, by stabilizing the explanation process across iterations.

\end{itemize}

\subsection{Other Related Works}


\paragraph{Graph Neural Networks.} Graph neural networks (GNNs) have emerged as a powerful tool for capturing and leveraging structural relationships in graph-structured data. Numerous neural architectures have been proposed in this domain, such as Graph Convolution Networks (GCN)~\cite{gcn}, Graph Attention Networks (GAT)~\cite{gat}, Message Passing Neural Networks (MPNNs)~\cite{mpnn}, etc. The fundamental idea behind these architectures is to iteratively aggregate information from each node's neighbors and update node and edge feature representations  accordingly, which are ultimately combined to perform predictions for graph reasoning tasks.

\paragraph{Post-hoc Explanation.}
Beside interpretable graph learning, another prominent approach for understanding GNNs is \emph{post-hoc explanation}, which aims to provide insights into a model’s predictions after training. A variety of post-hoc explanation methods have been proposed, including GNNExplainer~\cite{ying2019gnnexplainer}, GraphLIME~\cite{huang2022graphlime}, and PGExplainer~\cite{luo2020parameterized}.
These post-hoc methods adapt well to the non-IID nature of graph data by directly analyzing the graph structure, offering valuable insights. 

\paragraph{Beyond Self-Reflection: Mitigating LLM Hallucinations.}
Safety alignment via supervised/RLHF fine-tuning reduces hallucinations in practice (e.g., InstructGPT, Llama 2), though it can introduce an “alignment tax” and forgetting~\cite{DBLP:conf/nips/Ouyang0JAWMZASR22,DBLP:journals/corr/abs-2307-09288}. Other common approaches include orthogonal decoding-time methods and post-hoc verification with retrieval/tool-augmented editing: the former steer layer-contrastive logits or truth-correlated activations toward factual continuations without retraining~\cite{DBLP:conf/iclr/ChuangXLKGH24,DBLP:conf/naacl/ShiHLTZY24}, while the latter audit and revise drafts using external evidence or programmatic checks~\cite{DBLP:conf/iclr/GouSGSYDC24,DBLP:journals/corr/abs-2305-14002}. We leverage insights from reinforcement learning to guide the proposed method and will elaborate on them in Section~\ref{sec:ft}.

%% file: 20-pre.tex
\section{Preliminaries}\label{sec:pre}

This section formally introduces the problem of interpretable graph learning, provides the necessary background, and describes the challenging Spurious-Motif benchmark that serves as the focus of our experimental study. 

\subsection{Problem Definition}

The interpretable graph learning problem is formulated as follows. 
Given a graph reasoning task with an input graph $G = (V, E)$ and a target label $Y$ (e.g., in a graph classification task, the label $Y$ represents the category of the graph), let $S \subseteq G$ denote a subgraph. 
Define $I(S; Y)$ and $I(S; G)$ as the Shannon mutual information between the subgraph $S$ and the label $Y$, and between the subgraph $S$ and the original graph $G$, respectively. 
The goal is to find a subgraph $S$ with either bounded graph size or bounded $I(S; G)$ that maximizes $I(S; Y)$. 

These constraints collectively provide a principled way to extract important subgraphs as explanations, balancing interpretability and fidelity to the original graph. Formally,    
\begin{equation}
\label{eq:model}
    \begin{aligned}
    &\max_{S \subseteq G} \; I(S; Y) \\
    \text{s.t.} \quad & |S| \leq K \quad \text{or} \quad I(S; G) \leq \gamma~,
\end{aligned}
\end{equation}
where $K$ and $\gamma$ are given upper bounds. Note that, as in many prior works, the subgraph $S$ can be relaxed to allow fractional edges: if $z_e$ is used as an indicator for whether an edge $e$ is included in $S$, instead of restricting $z_e$ to be binary (0 or 1), we can let $z_e$ take any value in $[0, 1]$, representing the fraction of edge $e$ included in the subgraph $S$.

\subsection{L2X Architecture}
Most existing studies follow the L2X architecture proposed by~\cite{chen2018learning}, an interpretable graph learning framework built on graph neural networks and attention mechanisms.
In this architecture, there are two components: an \emph{upstream} module and a \emph{downstream} module. 
\begin{itemize}
    \item The upstream module evaluates the importance of each edge in the graph. Specifically, given a graph $G$, the upstream GNN $\vF(\cdot)$ predicts a fractional value $z_e \in [0, 1]$ for every edge $e \in E$, which can be interpreted as a score reflecting the importance of that edge. 
    \item Then, the downstream module utilizes them to predict the target label. The module uses the edge scores from the upstream module as an attention mask
    $\mask = \{z_e\}_{e \in E}$, which is element-wise multiplied with the original graph to produce a masked graph $G \odot \mask$. This masked graph is then fed into the downstream GNN $\vD(\cdot)$ to predict the target label $\pred$. 
Note that for the masked graph computation in the downstream GNN, when each node aggregates information from its neighbors, the contribution of each neighbor is weighted by the corresponding~$z_e$.

\end{itemize}





The architecture then trains both the upstream and downstream network parameters to minimize the combined loss:
\[\cL = \cL_{\text{up}}(\mask) + \cL_{\text{down}}(\pred)~. \]
The upstream module employs a loss $\cL_{\text{up}}(\mask)$ for the edge scores to ensure that the masked graph satisfies the given constraints, while the downstream loss $\cL_{\text{down}}(\pred)$  directly corresponds to the prediction loss for the target label. After training, the resulting $S = G \odot \mask$ is the returned subgraph.

\subsection{Spurious-Motif Benchmark}

The Spurious-Motif benchmark is a synthetic dataset designed to evaluate the robustness of interpretable graph learning methods in the presence of spurious correlations. Originally introduced by~\cite{wu2022discovering}, it has since been widely adopted in the interpretable graph learning and graph invariance learning literature~\cite{DBLP:DIR}. 

The benchmark comprises a series of datasets parameterized by a spurious correlation factor \( b \), with each dataset containing 18{,}000 graphs.
Each graph is constructed by combining one base structure and one motif. The base structures include \textit{Tree}, \textit{Ladder}, and \textit{Wheel}, denoted by \( S \in \{0, 1, 2\} \), while the motifs include \textit{Cycle}, \textit{House}, and \textit{Crane}, denoted by \( C \in \{0, 1, 2\} \). Given a bias parameter \( b \), each graph is generated by first selecting a motif of type \( C \), and then sampling a base type \( S \) according to the following distribution:
\[
P(S) = b \cdot \mathbb{I}(S = C) + \frac{1 - b}{2} \cdot \mathbb{I}(S \neq C),
\]
where \( \mathbb{I}(\cdot) \) denotes the indicator function.
In other words, if the base type matches the motif type (\( S = C \)), it is selected with probability \( b \); otherwise, one of the mismatched base types is selected with equal probability \( \frac{1 - b}{2} \). The final graph is formed by attaching a randomly sampled base structure of the selected type to the chosen motif.

The ground-truth label \( Y \) is determined solely by the motif type \( C \), making the motif the true causal factor. In contrast, the base type is a distractor that introduces spurious correlation with the label. As the bias parameter \( b \) increases, the alignment between base type and motif type strengthens, making it increasingly difficult for learning algorithms to distinguish true causal structures from spurious ones.

Following prior work~\cite{miao2022interpretable,DBLP:DIR}, our experiments use datasets where the training data are constructed with \( b = 0.5, 0.7, \) and \( 0.9 \), representing increasing levels of spurious bias. In the test data, however, we fix \( b = \frac{1}{3} \), so that base types are sampled independently of motif types. This setup poses two key challenges:
\begin{itemize}
  \item \textbf{Spurious Correlation in Training:} When the bias \( b \) is high, the base type becomes strongly correlated with the class label, even though it is not causally related. As a result, models tend to rely on the spurious base structure for prediction, which can mislead explanation methods into identifying irrelevant subgraphs.
  
  \item \textbf{Distribution Shift at Test Time:} Since the spurious correlation is removed in the test set (\( b = \frac{1}{3} \)), models that overfit to the base structures may suffer a significant drop in both accuracy and interpretability. This distribution shift places strong demands on explanation methods to distinguish causal features from misleading correlations and remain robust under changing data distributions.
\end{itemize}

%% file: 30-frame.tex
\section{Self-Reflection for Interpretable Graph Learning}\label{sec:framework}

This section investigates how self-reflection, a technique inspired by recent progress in large language models, can be applied to improve the interpretability of graph neural networks. Instead of modifying the model architecture or retraining the network, we propose a lightweight and training-free framework that enhances explanation quality by iteratively refining the model’s own interpretation outputs.


\begin{figure}[ht]
    \centering
    \includegraphics[width=0.95\linewidth]{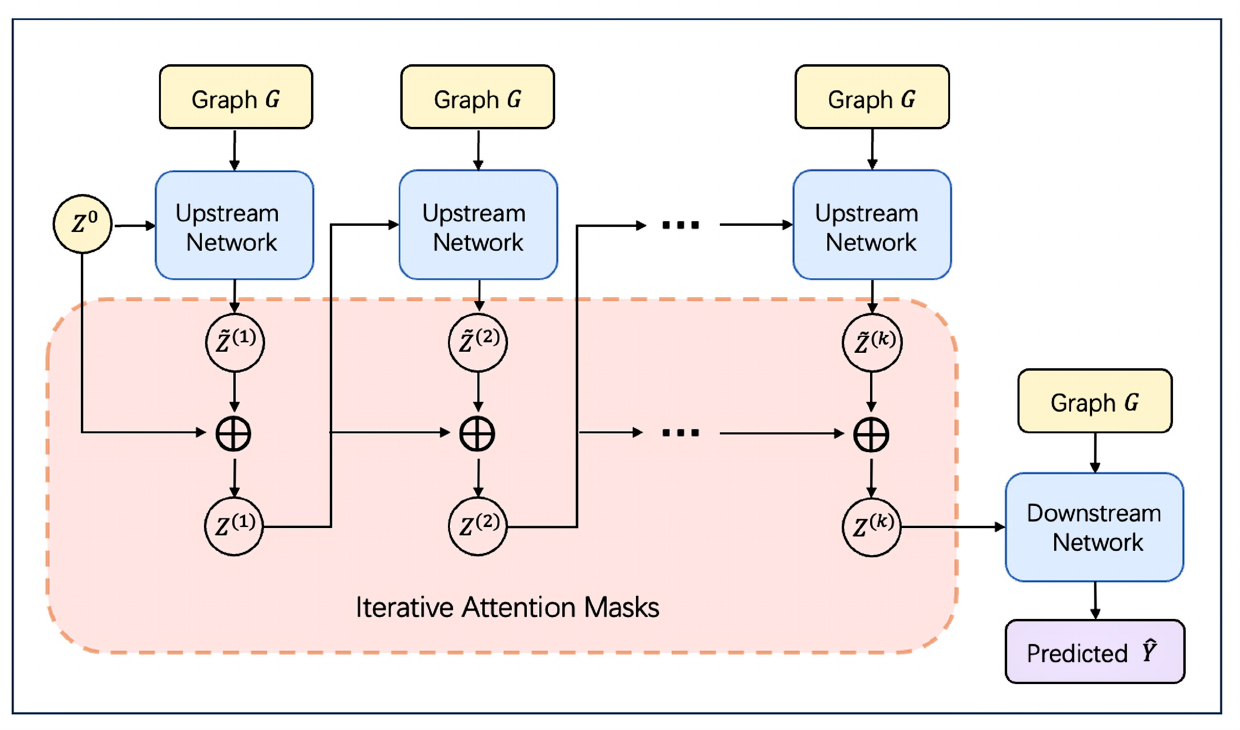} 
    \caption{An illustration of the self-reflection framework.}
    \label{fig:structure}
\end{figure}




We build upon the aforementioned L2X architecture to develop our framework. A key component of the framework is an iterative procedure that feeds the output of the upstream module back into itself. This self-reflective step allows the model to reassess the importance of each edge within a given subgraph. Notably, the model architecture remains unchanged throughout the process, and no additional training is required.

An illustration of the self-reflection framework is provided in~\cref{fig:structure}. In this framework, the upstream network is applied \( k \) times in an iterative manner. The process begins with an initial mask \( \mask^{(0)} := \{z_e^{(0)} = 1\}_{e \in E} \), which retains all edge information without any masking. 
At the \( t \)-th iteration, the masked graph \( G \odot \mask^{(t-1)} \) is passed to the upstream network \( \vF(\cdot) \), producing a new soft mask \( \tilde{\mask}^{(t)} \):
\[
\tilde{\mask}^{(t)} = \vF\left(G \odot \mask^{(t-1)}\right)~.
\]
This mask \( \tilde{\mask}^{(t)} \) reflects the importance of edges within the already masked graph \( G \odot \mask^{(t-1)} \). Consequently, the updated important subgraph at this iteration is effectively represented as \( \left(G \odot \mask^{(t-1)}\right) \odot \tilde{\mask}^{(t)} \). 
To preserve this structure, we update the mask\footnote{An ablation study of this design is provided in the appendix.} by element-wise multiplication:
\[
\mask^{(t)} = \tilde{\mask}^{(t)} \cdot \mask^{(t-1)}~,
\]
and use the newly masked graph \( G \odot \mask^{(t)} \) as input for the next iteration.
 After $k$ iterations, the final mask $\mask^{(k)}$ is obtained and passed to the downstream module to predict the target label. 

\paragraph{Monotonicity Remark.}
We note that this element-wise update scheme naturally induces a monotonicity property in the mask across iterations. Specifically, since each new mask is obtained by multiplying the current prediction with the previous mask, the importance score assigned to each edge cannot increase over time. 
Such monotonicity is important for maintaining stability in the self-reflection process. It prevents the upstream module from assigning high importance to edges that have already been substantially downweighted in earlier iterations, thereby reducing the risk of relying on noisy or irrelevant structures. This is analogous to the overestimation issue in offline Q-learning~\cite{DBLP:conf/icml/FujimotoMP19}, where high values may be assigned to unseen or unsupported states. In contrast, our framework suppresses such behavior by construction, leading to more consistent and reliable explanations.

\begin{figure*}[tb]
    \centering
    \subfigure[Accuracy across self-reflection iterations.\label{fig:acc}]{
        \includegraphics[width=0.76\linewidth]{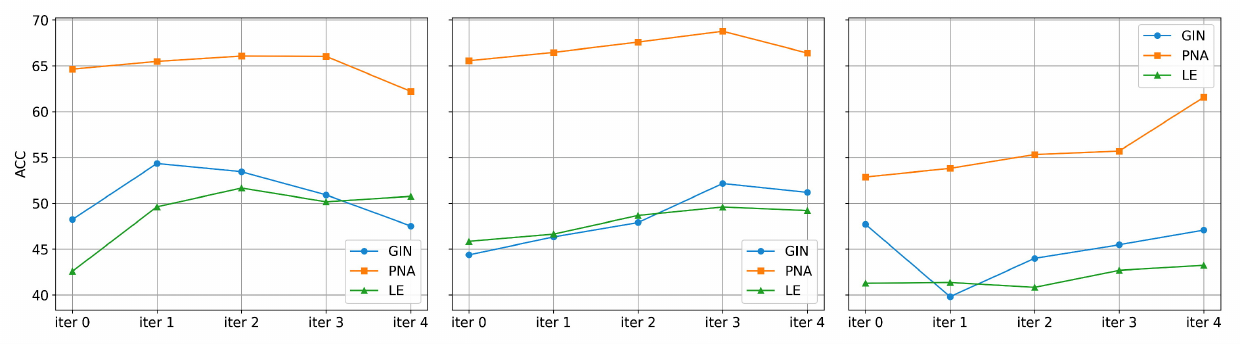}
        \vspace{1em}
    }

    \subfigure[AUC across self-reflection iterations.\label{fig:roc}]{
        \includegraphics[width=0.76\linewidth]{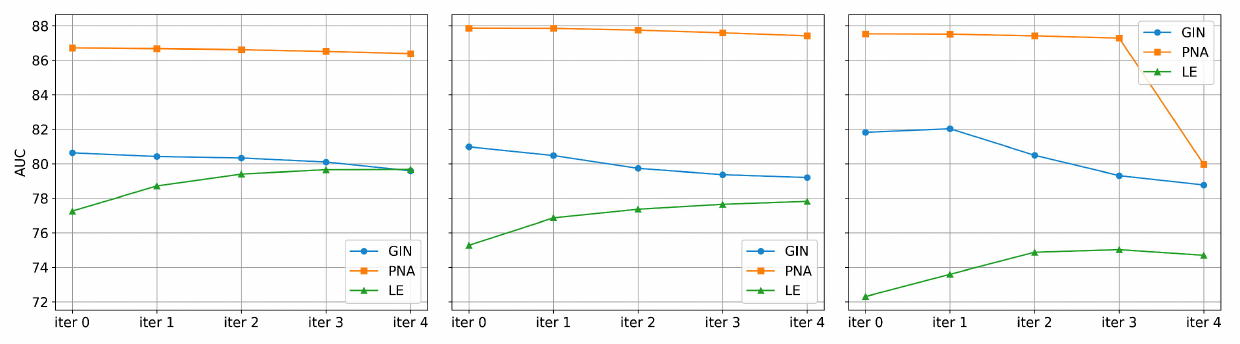}
    }

    \caption{Performance trends under the self-reflection framework. From left to right, the plots correspond to datasets with spurious correlation levels \( b = 0.5 \), \( 0.7 \), and \( 0.9 \), respectively.}
    \label{fig:performance}
\end{figure*}

\begin{figure*}[tb]
    \centering
    \includegraphics[width=0.77\linewidth]{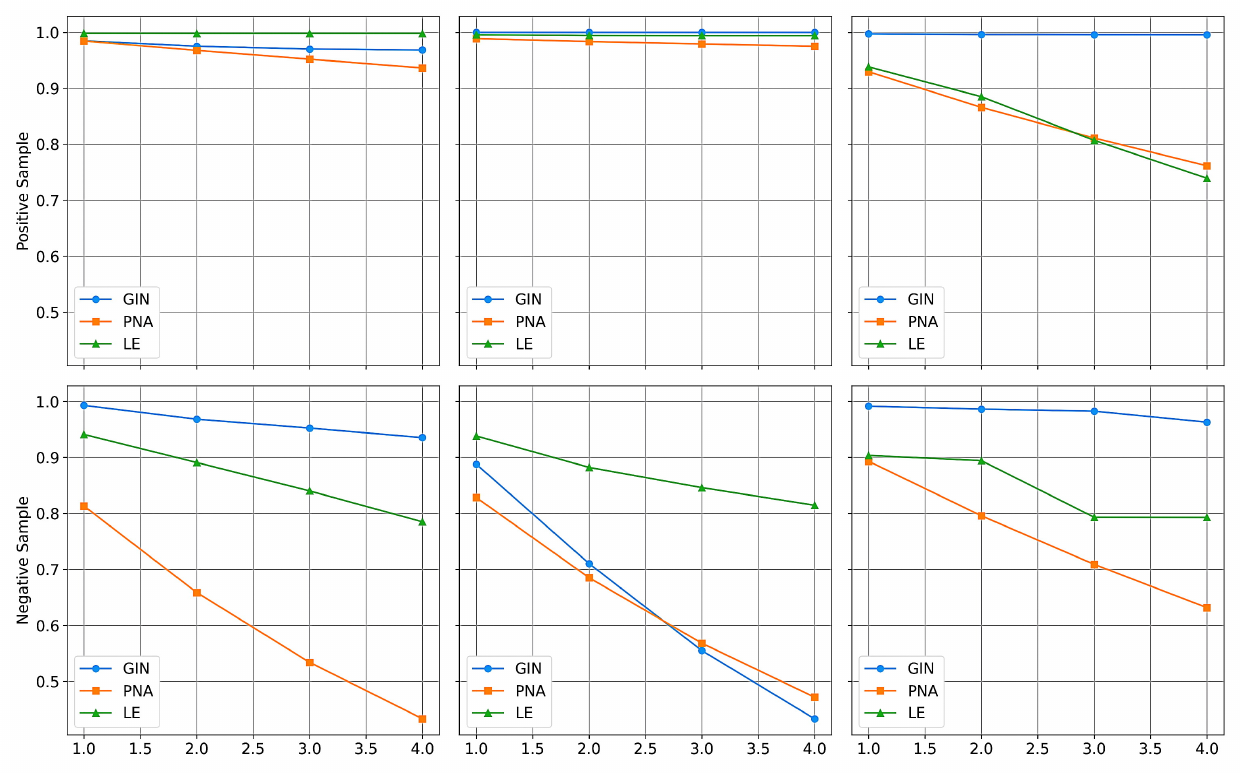}
    \caption{An illustration of how edge importance scores evolve across self-reflection iterations. The top row shows the average edge scores for positive samples, while the bottom row shows those for negative samples. From left to right, the plots correspond to datasets with spurious correlation levels \( b = 0.5 \), \( 0.7 \), and \( 0.9 \), respectively.}
    \label{fig:score}
\end{figure*}

\section{Empirical Study of Self-Reflection}~\label{sec:framework_exp}


In this section, we conduct an empirical investigation of the behavior and effectiveness of the proposed self-reflection framework. Our goal is to evaluate how iterative self-reflection influences both prediction performance and explanation quality in interpretable graph learning. In particular, we aim to understand the following questions:

\begin{itemize}
    \item Does the framework consistently improve task accuracy across iterations?
    \item Do the generated explanations become more focused and stable over time?
\end{itemize}

We begin by outlining the experimental setup, followed by a detailed presentation and analysis of the results.

\subsection{Setup}

In our experiments, we apply the self-reflection framework to a recently popular interpretable graph learning method: \gsat\cite{miao2022interpretable}, using several GNN backbone architectures that are widely adopted in interpretable learning, including GIN~\cite{DBLP:conf/iclr/XuHLJ19}, PNA~\cite{DBLP:conf/nips/CorsoCBLV20}, and LE~\cite{DBLP:conf/aaai/RanjanST20}.


\paragraph{Datasets.} 
As mentioned earlier, our experiments are primarily conducted on the Spurious-Motif benchmark, using the spurious correlation parameter \( b \in \{0.5, 0.7, 0.9\} \), following the standard setting in~\cite{miao2022interpretable}. We note that similar experiments have also been conducted on several other widely used benchmarks, including BA-2Motifs~\cite{luo2020parameterized}, Mutag~\cite{debnath1991structure}, MolHIV~\cite{wu2018moleculenet}, MolBACE~\cite{wu2018moleculenet}, and MolBBBP~\cite{wu2018moleculenet}. Due to space constraints, these additional results are presented in the appendix.

\paragraph{Evaluation Metrics.} Following prior work, we evaluate performance using two metrics: prediction accuracy on the downstream classification task and the area under the ROC curve (AUC). To clarify, given a set of predicted edge importance scores, AUC measures the probability that a randomly chosen important edge is assigned a higher score than a randomly chosen unimportant edge.

\paragraph{Computational Details.} 
We first train each model using the standard hyperparameter configuration provided in~\cite{miao2022interpretable}. The trained networks are then integrated into our self-reflection framework without further parameter updates. All experiments are conducted on a machine equipped with an NVIDIA RTX 4090 GPU, and the reported results are averaged over four independent runs.

\subsection{Results and Analysis}

We evaluate the performance of the training-free self-reflection framework across different GNN backbone architectures (GIN, PNA, and LE) under varying levels of spurious correlation (\(b = 0.5, 0.7, 0.9\)). The results are reported in terms of classification accuracy and AUC over multiple reflection iterations, as shown in~\cref{fig:performance}.

\Cref{fig:acc} shows that the self-reflection framework consistently improves classification accuracy across most settings, particularly under higher levels of spurious correlation. Accuracy generally increases over the first few iterations, indicating that the model benefits from refining its attention masks. This suggests that iterative refinement effectively filters out spurious edges, allowing the model to focus more on causally relevant substructures.

\paragraph{Diminishing Returns of Self-Reflection.} We also observe that, after a certain number of iterations, performance may begin to decline slightly. Interestingly, the iteration at which this turning point occurs varies with the spurious correlation parameter \( b \); in general, the higher the value of \( b \), the later the peak performance is reached. This behavior resembles a known phenomenon~\cite{DBLP:journals/corr/abs-2412-21187} in large language models, where excessive self-reflection on simple problems can lead to degraded outputs. In such cases, further reflection introduces unnecessary modifications or noise, ultimately harming performance rather than improving it.

\paragraph{Divergence Between Accuracy and AUC.}\Cref{fig:roc} shows that, in contrast to accuracy, AUC remains relatively stable across iterations. Interestingly, we also observe a phenomenon previously noted in prior work~\cite{miao2022interpretable}: accuracy and AUC can sometimes exhibit diverging trends. This is most notable in the case of PNA on the dataset with \( b = 0.9 \), where the accuracy increases sharply at iteration 4, while the AUC drops.


This divergence suggests a shift in the model’s behavior during self-reflection. One possible explanation is that the model becomes more confident in a smaller subset of edges, which may strengthen its prediction ability while reducing the spread of importance scores across the full edge set. As a result, AUC, which is a ranking-based metric over all edges, may decline even when the model is focusing more effectively on task-relevant structures.

To examine this hypothesis, we sampled a subset of edges from the dataset and tracked how their importance scores evolve across iterations. The results are shown in~\cref{fig:score}, where positive samples refer to edges labeled as important in the ground truth, and negative samples refer to those deemed unimportant.
We observe that the scores of positive edges remain consistently high across iterations. In contrast, the scores of negative edges decrease rapidly with each iteration, indicating that the model becomes increasingly confident in suppressing irrelevant or spurious edges.

Notably, some edges that may have initially received moderate scores are gradually de-emphasized over time. This effect is more pronounced for negative samples, but also appears in a small fraction of weakly important positive edges. These observations support our hypothesis that self-reflection encourages the model to concentrate attention on a more selective set of high-confidence edges, which can benefit prediction performance while altering the global ranking distribution used by AUC.


%% file: 35-tun.tex
\section{Enhancing Performance via Fine-Tuning}\label{sec:ft}

\begin{table*}[tb]
    \centering
    \footnotesize
    \caption{Performance of different methods. The best result is shown in bold, and the second-best is underlined.
    }
    \begin{tabular}{l c c c c c c c}
    \toprule
    \textbf{Methods} & \multicolumn{3}{c}{\textbf{ACC}} & \multicolumn{3}{c}{\textbf{AUC}} \\
    \cmidrule(lr){2-4} \cmidrule(lr){5-7}
     & \textbf{0.5} & \textbf{0.7} & \textbf{0.9} & \textbf{0.5} & \textbf{0.7} & \textbf{0.9} \\
    \midrule
GNNExplainer&-&-&-&62.62{\tiny$\pm$1.35}&62.25{\tiny$\pm$3.61}&58.86{\tiny$\pm$1.93}\\
PGExplainer&-&-&-&69.54{\tiny$\pm$5.64}& {72.33}{\tiny$\pm$9.18}&\text{72.34}{\tiny$\pm$2.91}\\
DIR&45.50{\tiny$\pm$2.15}&43.36{\tiny$\pm$1.64}&39.87{\tiny$\pm$0.56}&78.15{\tiny$\pm$1.32}&77.68{\tiny$\pm$1.22}&49.08{\tiny$\pm$3.66}\\
\hline
GSAT+LE&42.57{\tiny$\pm$1.98}&45.84{\tiny$\pm$2.74}&41.27{\tiny$\pm$1.39}&77.26{\tiny$\pm$1.62}&75.27{\tiny$\pm$1.60}&72.30{\tiny$\pm$1.60}\\ 
SR+LE &51.66{\tiny$\pm$7.64}&48.67{\tiny$\pm$7.40}&40.82{\tiny$\pm$3.04}&79.40{\tiny$\pm$0.77}&77.36{\tiny$\pm$1.27}&\underline{74.88}{\tiny$\pm$2.56}\\
FT-SR+LE &\textbf{54.39}{\tiny$\pm$6.86}&\textbf{57.98}{\tiny$\pm$5.73}&\underline{43.45}{\tiny$\pm$3.76}&\underline{81.45}{\tiny$\pm$1.39}&\underline{79.94}{\tiny$\pm$2.25}&73.82{\tiny$\pm$4.28}\\
\hline
GSAT+GIN&52.74{\tiny$\pm$4.08}&49.12{\tiny$\pm$3.29}&\textbf{44.22}{\tiny$\pm$5.57}&78.45{\tiny$\pm$3.12}&74.07{\tiny$\pm$5.28}&71.97{\tiny$\pm$4.41}\\    
SR+GIN &53.44{\tiny$\pm$2.74}&47.90{\tiny$\pm$3.47}&43.99{\tiny$\pm$1.08}&80.33{\tiny$\pm$1.10}&79.73{\tiny$\pm$1.17}&\textbf{80.49}{\tiny$\pm$2.34}\\
FT-SR+GIN &\underline{54.11}{\tiny$\pm$2.26}&\underline{49.16}{\tiny$\pm$4.98}&42.47{\tiny$\pm$5.10}&\textbf{84.41}{\tiny$\pm$1.08}&\textbf{84.00}{\tiny$\pm$3.42}&80.07{\tiny$\pm$1.77}\\
\hline

\hline
    \end{tabular}
    \label{tab:OverallRoc}
\end{table*}

\begin{table*}[tb]
    \centering
    \footnotesize
    \caption{ACC Performance of different methods.The best result is shown in bold, and the second-best is underlined.
    }
    \begin{tabular}{l c c c c c c c c}
    \toprule
    \textbf{Methods} &\textbf{BA\_2motif}&\textbf{Mutag}&\textbf{Molbace}&\textbf{Molbbbp}&\textbf{Molhiv}  \\ \hline

GSAT+LE     &81.75 {\tiny$\pm$9.48}&90.78	{\tiny$\pm$1.23}&\textbf{75.52}	{\tiny$\pm$1.43}&55.45	{\tiny$\pm$1.89}    &76.75	{\tiny$\pm$0.85}    \\
SR+LE       &89.25 {\tiny$\pm$5.55}&87.32	{\tiny$\pm$2.03}&69.40	{\tiny$\pm$0.99}&54.28	{\tiny$\pm$1.10}    &96.81	{\tiny$\pm$0.09}\\
FT-SR+LE    &\underline{94.00} {\tiny$\pm$7.78}&90.37	{\tiny$\pm$0.17}    &60.52	{\tiny$\pm$1.77}&56.04	{\tiny$\pm$1.47}&96.82	{\tiny$\pm$0.18}\\
\hline
GSAT+GIN    &\textbf{100.00} {\tiny$\pm$0.00}&91.38	{\tiny$\pm$0.61}&\underline{73.83}	{\tiny$\pm$1.65}&\textbf{62.51}	{\tiny$\pm$2.23}        &77.59	{\tiny$\pm$0.98}\\
SR+GIN      &\textbf{100.00} {\tiny$\pm$0.00}&\underline{91.55}	{\tiny$\pm$1.32}&68.25	{\tiny$\pm$1.28}&56.98	{\tiny$\pm$1.56}    &\textbf{96.88}	{\tiny$\pm$0.06}\\
FT-SR+GIN   &\textbf{100.00} {\tiny$\pm$0.00}&\textbf{92.36}	{\tiny$\pm$0.49}&61.83	{\tiny$\pm$3.32}&\underline{57.67}	{\tiny$\pm$1.25}    &\underline{96.85}	{\tiny$\pm$0.10}\\
\hline
\end{tabular}
\label{tab:OverallOtherDatasets}
\end{table*}

While the self-reflection framework improves interpretability in a post-hoc and training-free manner, we now explore whether its performance can be further enhanced through fine-tuning. This section introduces a fine-tuning strategy specifically designed for the framework, aiming to adapt the model more effectively to the iterative reasoning process.

\subsection{Rethinking the Training Objective}

The original \gsat method adopts a one-step explanation loss, which assumes that importance scores are computed based on a fixed input graph. However, under the self-reflection framework, the input graph evolves over iterations, and each round of prediction depends on the masked subgraph produced in the previous step. This mismatch renders the original loss formulation misaligned with the reflective inference process. Empirical evidence supporting this claim is provided in the appendix, where we show that directly reusing the original loss fails to yield consistent improvements when applied to the self-reflection framework.
This motivates the need for a new training objective that can encourage more consistent edge-level reasoning across iterations and improve downstream performance. 

We notice that in the self-reflection framework, each mask \( \mask^{(t)} \) is derived from the upstream network’s prediction over the masked graph \( G \odot \mask^{(t-1)} \). This recursive structure suggests that later masks are conditionally dependent on earlier ones and should ideally exhibit a form of semantic consistency. That is, if an edge is consistently marked as important across iterations, it should be strongly supported by the model’s internal representation. Conversely, fluctuations or instability in importance scores may indicate uncertainty or spurious attribution.

This perspective motivates the design of a training objective that explicitly encourages cross-iteration consistency in edge importance estimation. In what follows, we formalize this intuition by translating it into a theoretical formulation.


\subsection{Thereotical Insight}

We formulate the core optimization problem underlying the framework as follows. Without loss of generality, we assume that each subgraph is constrained by its mutual information with the original graph, i.e., \( I(S; G) \leq \gamma \).
\begin{equation}
\label{eq:rif}
    \begin{aligned}
    \max \; I(G\odot \mask^{(k)};& Y) \\
    \text{s.t.}  \qquad I(G\odot \mask^{(t)}; G) &\leq \gamma \qquad  \forall t\in [k], \\
     \quad \vF(G\odot \mask^{(t-1)}) \cdot  \mask^{(t-1)} &= \mask^{(t)} \;\;\; \forall t\in [k], \\
     \qquad \quad \;\;\;\;\; z_e^{(t)} \in &\; [0,1] \;\;\quad  \forall t\in [k],e\in E~. \\
\end{aligned}
\end{equation}

This formulation generalizes Problem~\eqref{eq:model}. The goal remains to maximize the mutual information between the final masked graph and the target label \( Y \). The first constraint is a natural extension of the one used in Problem~\eqref{eq:model}, enforcing an information bottleneck at each iteration. The second constraint is newly introduced by the self-reflection framework and captures the recursive relationship between masks across iterations. It ensures that each mask is generated by applying the upstream model to the previous mask.

We define the notation \( Z \succeq Z' \) to indicate that \( z_e \geq z_e' \) for all \( e \in E \). A mask sequence \( \{Z^{(t)}\}_{t \in [k]} \) is said to be \emph{monotone} if it satisfies \( Z^{(1)} \succeq Z^{(2)} \succeq \cdots \succeq Z^{(k)} \). Clearly, the sequence of masks generated by the self-reflection framework is monotone by construction. 

We now analyze the properties of the optimal masks. For simplicity, we assume in the theoretical analysis that the neural network \( \vF(\cdot) \) is sufficiently expressive to realize any monotone mask sequence through appropriate parameterization. The following theorem provides formal support for the consistency intuition introduced earlier.


\begin{theorem}\label{thm:consistent}
    There always exists a set of optimal masks $\{\mask^{(t)}\}_{t \in [k]}$ to Problem~\eqref{eq:rif} that maintains consistency, i.e., $\mask^{(1)} = \mask^{(2)} = \cdots = \mask^{(k)}$.
\end{theorem}

\begin{proof}

We prove this theorem by leveraging the strong connection between Problem~\eqref{eq:model} and Problem~\eqref{eq:rif}. The basic idea is to first derive an upper bound for the optimal objective value of Problem~\eqref{eq:rif} using an optimal solution of Problem~\eqref{eq:model}. Then, we construct a consistent feasible solution whose objective achieves this upper bound, thereby demonstrating that it is indeed an optimal solution.


Let $Z^*$ denote the mask such that $G \odot Z^*$ is an optimal solution to Problem~\eqref{eq:model}. Consequently, we have 
\begin{equation}
\label{eq:opt_constrain}
I(G \odot Z^*; G) \leq \gamma~,
\end{equation}
and for any $Z$ satisfying $I(G \odot Z; G) \leq \gamma$,
it holds that 
\begin{equation}
\label{eq:opt_obj}
I(G \odot Z, Y) \leq I(G \odot Z^*, Y)~.
\end{equation}

\cref{eq:opt_obj} implies that any feasible solution to Problem~\eqref{eq:rif} has an objective value of at most $I(G \odot Z^*; Y)$, because any feasible $Z^{(k)}$ must satisfy the mutual information upper bound. Therefore, if we can construct a feasible solution to Problem~\eqref{eq:rif} with an objective value equal to $I(G \odot Z^*; Y)$, it must be an optimal solution.

To this end, we construct a set $\cZ$ of $k$ masks by setting each $Z^{(t)} = Z^*$. \cref{eq:opt_constrain} and the flexibility assumption of network $\vF(\cdot)$ guarantee that $\cZ$ is a feasible solution to Problem~\eqref{eq:rif}, as it is monotone and every $Z^{(t)}$ satisfies the mutual information upper bound. Furthermore, since $Z^{(k)} = Z^*$, the objective value of this feasible solution is $I(G \odot Z^*; Y)$, thereby completing the proof that Problem~\eqref{eq:rif} admits a consistent optimal solution.
\end{proof}

The theorem above suggests that a well-performing network within the self-reflection framework must satisfy a form of self-consistency. Moreover, the analysis used in the proof offers insight into why the framework leads to empirical improvements. Unlike the original L2X architecture, where the model attempts to identify important edges in a single step, the self-reflection framework adopts a progressive approach that iteratively filters out irrelevant edges. Such a refinement enhances the interpretation process by promoting more stable and reliable importance estimation.

\subsection{Fine-Tuning Objective}




The theoretical results above demonstrate the existence of a mask-consistent optimal solution to Problem~\eqref{eq:rif}. Therefore, introducing consistency constraints on the masks can substantially reduce the feasible solution space without affecting the optimal objective value. 

Motivated by this observation, we refine the training process of the self-reflection framework by encouraging the masks to be as consistent as possible across iterations.
To implement this idea, we introduce a mask consistency loss:
\[
\cL_{\text{con}} (\{\mask^{(t)}\}_{t \in [k]}) = \frac{2}{k(k-1)} \cdot \sum_{1 \leq t < t' \leq k} \left\| \mask^{(t)} - \mask^{(t')} \right\|_1,
\]
where \( \mask^{(t)} - \mask^{(t')} \) denotes the element-wise difference between two masks, and \( \| \cdot \|_1 \) denotes the \( \ell_1 \)-norm.

This loss term computes the average pairwise \( \ell_1 \)-distance across all mask pairs, effectively measuring the overall inconsistency among masks. By minimizing this term, we promote cross-iteration consistency in edge importance estimation. Combining this loss with the downstream prediction loss yields the overall fine-tuning objective:
\[
\cL_{\text{fine-tune}} = \cL_{\text{con}} (\{\mask^{(t)}\}_{t \in [k]}) + \cL_{\text{down}}(\pred),
\]
where \( \cL_{\text{down}} \) denotes the standard prediction loss on the downstream task.

\subsection{Empirical Evaluation}

We evaluate our fine-tuning strategy through empirical experiments. Specifically, we focus on the self-reflection framework with iteration \( k = 2 \), and choose LE and GIN as the backbone networks, as both exhibit noticeable performance changes at this iteration depth. We use the Adam optimizer with a learning rate of \( 1 \times 10^{-4} \) and a batch size of 512, and fine-tune the models for 10 epochs. 

The results, along with comparisons to existing methods including GNNExplainer~\cite{ying2019gnnexplainer}, PGExplainer~\cite{DBLP:conf/nips/LuoCXYZC020}, and DIR~\cite{DBLP:DIR}, are presented in~\cref{tab:OverallRoc}. In the table, our training-free self-reflection framework is denoted as SR, while the fine-tuned version is denoted as FT-SR.

From the table, we observe that fine-tuning indeed leads to improved performance in most cases, particularly in terms of AUC. This aligns well with our earlier analysis of AUC degradation: by introducing a mask consistency loss, fine-tuning helps mitigate the issue of weakly important edges being down-weighted too aggressively during the iterative process. As a result, the model achieves better AUC while maintaining high accuracy.

In \cref{tab:OverallOtherDatasets} we extend the comparison to additional datasets, reporting accuracy and, for \texttt{molbbbp}, \texttt{molbace}, and \texttt{molhiv}, per-batch AUC to account for class imbalance (following prior work). Two patterns emerge. \textbf{First}, on \texttt{molhiv} both SR and FT-SR deliver substantial gains, consistent with our claim that self-reflection mitigates overfitting to spurious shortcuts: in this dataset the training accuracy typically exceeds validation/test markedly, and our methods narrow this generalization gap. \textbf{Second}, on datasets where spurious correlations are weak or absent, we observe small accuracy decreases—reflecting a trade-off wherein iterative masking and consistency regularization sacrifice a bit of signal to gain robustness. Practically, SR/FT-SR is most beneficial in settings prone to shortcut features; otherwise, the consistency weight may need to be reduced or the method applied selectively.

\section{Ablation Study}
This section presents ablations showing that the effectiveness of our self-reflection framework.
\subsection{Without Mask Multiplication}
\begin{figure*}[tbh]
    \centering
    \subfigure[Accuracy across self-reflection iterations without multiplication.\label{fig:appendacc}]{
        \includegraphics[width=0.76\linewidth]{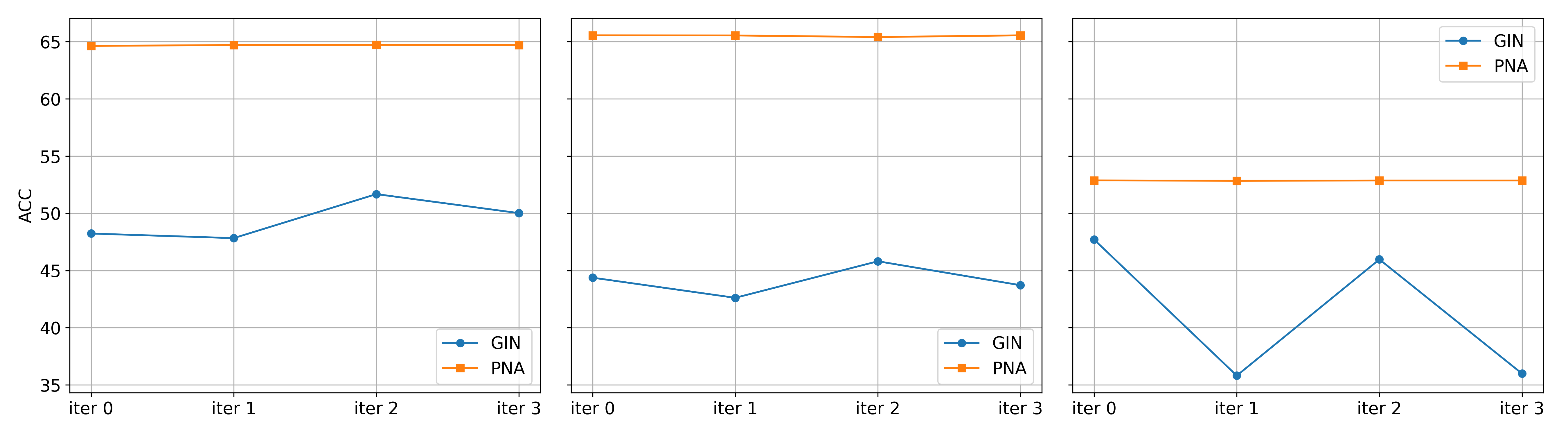}
        \vspace{1em}
    }

    \subfigure[AUC across self-reflection iterations without multiplication.\label{fig:appendroc}]{
        \includegraphics[width=0.76\linewidth]{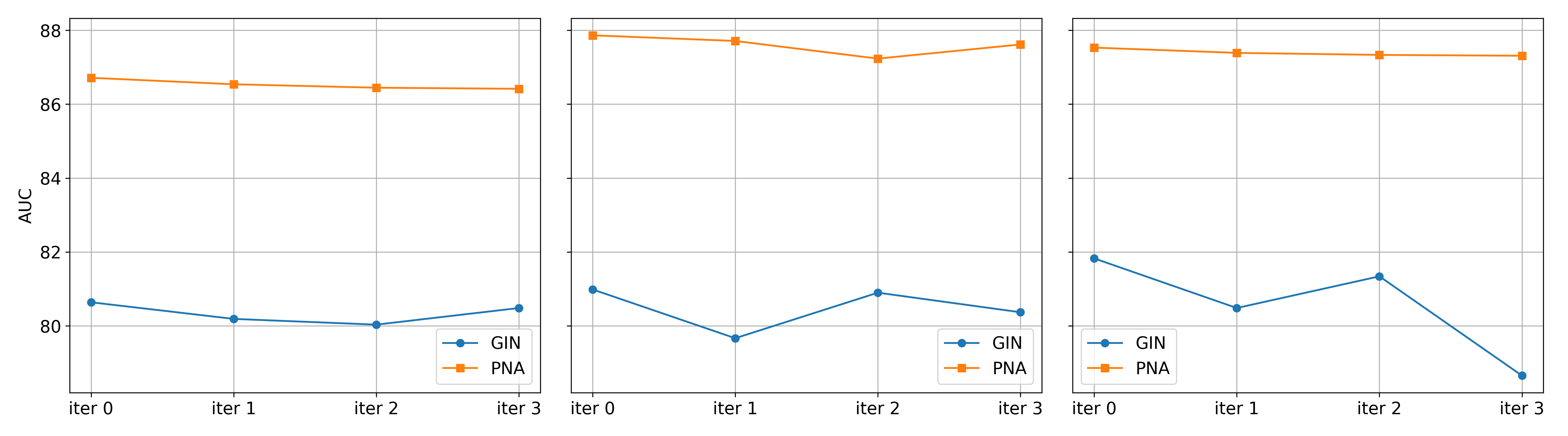}
    }

    \caption{Performance trends under the self-reflection framework without multiplication. From left to right, the plots correspond to datasets with spurious correlation levels \( b = 0.5 \), \( 0.7 \), and \( 0.9 \), respectively.}
    \label{fig:performanceOfNoMulti}
\end{figure*}

\begin{table*}[!t]
    \centering
    \footnotesize  
\caption{Performance comparison of the fine-tuning objective, our designed objective, and the original GSAT loss (denoted as *\textsubscript{raw}).
    }
    \begin{tabular}{l c c c c c c c}
    \toprule
    \textbf{Methods} & \multicolumn{3}{c}{\textbf{ACC}} & \multicolumn{3}{c}{\textbf{AUC}} \\
    \cmidrule(lr){2-4} \cmidrule(lr){5-7}
     & \textbf{0.5} & \textbf{0.7} & \textbf{0.9} & \textbf{0.5} & \textbf{0.7} & \textbf{0.9} \\
    \midrule

FT-SR+LE &{54.39}{\tiny$\pm$6.86}&{57.98}{\tiny$\pm$5.73}&{43.45}{\tiny$\pm$3.76}&{81.45}{\tiny$\pm$1.39}&{79.94}{\tiny$\pm$2.25}&73.82{\tiny$\pm$4.28}\\

FT-SR+LE{\tiny raw} &50.36	{\tiny$\pm$7.41}&52.13	{\tiny$\pm$4.98}&40.15	{\tiny$\pm$2.85}&77.90	{\tiny$\pm$1.35}&77.24	{\tiny$\pm$1.73}&73.33	{\tiny$\pm$2.38}\\
\hline
FT-SR+GIN &{54.11}{\tiny$\pm$2.26}&{49.16}{\tiny$\pm$4.98}&42.47{\tiny$\pm$5.10}&{84.41}{\tiny$\pm$1.08}&{84.00}{\tiny$\pm$3.42}&80.07{\tiny$\pm$1.77}\\

FT-SR+GIN{\tiny raw} 
&53.10	{\tiny$\pm$0.38}&47.50	{\tiny$\pm$4.83}&44.43	{\tiny$\pm$0.47}&80.07	{\tiny$\pm$0.94}&80.73	{\tiny$\pm$1.47}&82.28	{\tiny$\pm$1.77}\\
\hline

\hline
FT-SR+PNA &69.16{\tiny$\pm$1.41}&69.81{\tiny$\pm$1.35}&59.44{\tiny$\pm$1.04}&82.58{\tiny$\pm$2.91}&84.06{\tiny$\pm$3.19}&81.12{\tiny$\pm$3.06}\\

FT-SR+PNA{\tiny raw} &63.58	{\tiny$\pm$2.03}&65.58	{\tiny$\pm$1.18}&56.74	{\tiny$\pm$1.45}&86.58	{\tiny$\pm$0.76}&84.26	{\tiny$\pm$2.29}&84.58	{\tiny$\pm$0.78}\\
\hline
    \end{tabular}
    \label{tab:GSATRawLossComparison}
\end{table*}
To evaluate the role of the multiplicative update mechanism in our self-reflection framework, we conduct an ablation study where the edge importance masks are updated \emph{additively} or \emph{replaced directly} at each iteration, instead of being accumulated via element-wise multiplication. That is, instead of computing the new mask as $Z^{(t)} = Z^{(t-1)} \cdot \tilde{Z}^{(t)}$, we use $Z^{(t)} = \tilde{Z}^{(t)}$.

Figure~\ref{fig:appendacc} and Figure~\ref{fig:appendroc} illustrate the performance trends of this variant in terms of classification accuracy and explanation AUC over multiple reflection iterations.

We observe that, across all datasets and backbones, this variant exhibits either a \textbf{flat} or \textbf{gradually declining} trend in both metrics. In contrast to the standard multiplicative approach, this version shows \textbf{no consistent performance gain} through self-reflection. This suggests that the multiplicative mechanism plays a crucial role in progressively filtering out spurious or irrelevant edges, and helps maintain stability in the refinement process. Without it, the model tends to oscillate or overwrite useful information learned in earlier iterations.

These results highlight the necessity of enforcing monotonic importance suppression through multiplicative updates to ensure the effectiveness of self-reflection in interpretability tasks.


\subsection{Fine-Tuning with Original GSAT Loss}

In this ablation study, we compare the performance of our proposed consistency-regularized fine-tuning objective with the original GSAT loss (denoted as GSAT\textsubscript{raw}) under the same iterative framework. As shown in Table~\ref{tab:GSATRawLossComparison}, the results highlight that the original GSAT loss does not produce favorable results when applied in a fine-tuning scenario.

%% file: 60-con.tex
\section{Conclusion}\label{sec:con}


In this paper, we propose a self-reflection framework for interpretable graph learning, which iteratively refines explanation masks without requiring additional training. We provide theoretical insights into the structure of optimal solutions and further introduce a consistency-based fine-tuning strategy to enhance performance. Experiments demonstrate the effectiveness of our approach.

There remain several interesting directions for future work. For example, one could explore alternative fine-tuning objectives that may further enhance performance. In particular, our current fine-tuning objective is deliberately dataset-agnostic and fully offline—decoupled from instance-level interaction—which forfeits RL-style exploration over the space of masking policies; designing exploration-aware, RL-inspired objectives within this framework is a promising direction. Therefore, new objective functions or mechanisms need to be further studied to mitigate such trade-offs.

%% file: 99-app.tex
\appendix

\section*{Appendix}

This appendix provides additional insights, experiments, and ablations that supplement the main paper, with outline following.

\subsection*{Outline}
    \paragraph{1. Overview of the GSAT Method} 
    A detailed description of the GSAT framework~\cite{miao2022interpretable}, its training mechanism, and how it integrates with our self-reflection procedure.
    
    \paragraph{2. Dataset Descriptions} 
    Summary of the statistics, label semantics, and motif structures for the datasets used: Spurious-Motif, MUTAG, BA-2Motifs, MolBBBP, MolBACE, and MolHIV.
    
    \paragraph{3. Experimental Results on PNA Backbone and Additional Datasets} 
    Full experimental results on additional datasets under various levels of spurious correlation, including results using PNA as the GNN backbone.

\subsection*{A.1 Detailed Description about GSAT}

GSAT is an interpretable GNN framework with a clear L2X architecture. In this section, we provide a more detailed mathematical formulation of GSAT.

\textbf{Objective Function.}  
The GSAT framework follows the general L2X objective function\cite{chen2018learning} and enforces the constraint:

\[
I(G_S; G) \leq \gamma.
\]

To incorporate this constraint into optimization, GSAT applies \textbf{Lagrangian relaxation}, modifying the original constrained problem into an unconstrained form:

\[
\max _{\phi} \left( I(G_S; Y) - \beta I(G_S; G) \right), \quad \text{s.t.} \quad G_S \sim g_{\phi}(G).
\]

Here, \( g_{\phi} \) represents the upstream model that generates the subgraph \( G_S \).

\textbf{Lower Bound Approximation.}  
The first term in the objective function, \( I(G_S; Y) \), measures how much information the selected subgraph \( G_S \) preserves about the target label \( Y \). By definition:

\[
I(G_S; Y) = \mathbb{E}_{G_S, Y} \left[ \log \frac{\mathbb{P}(Y \mid G_S)}{\mathbb{P}(Y)} \right].
\]

Since computing \( \mathbb{P}(Y \mid G_S) \) exactly is intractable, GSAT introduces a \textbf{variational lower bound} using an approximate posterior \( \mathbb{P}_{\theta}(Y \mid G_S) \), leading to the following decomposition, here \(\theta\) can be viewed as the parameters of downstream model:

\[
\begin{aligned}
I(G_S; Y) &= \mathbb{E}_{G_S, Y} \left[ \log \frac{\mathbb{P}_{\theta}(Y \mid G_S)}{\mathbb{P}(Y)} \right]  \\& + \mathbb{E}_{G_S} \left[ \operatorname{KL} \left( \mathbb{P}(Y \mid G_S) \,\Vert\, \mathbb{P}_{\theta}(Y \mid G_S) \right) \right] \\
&\geq \mathbb{E}_{G_S, Y} \left[ \log \frac{\mathbb{P}_{\theta}(Y \mid G_S)}{\mathbb{P}(Y)} \right].
\end{aligned}
\]

Since \( H(Y) \) is a constant with respect to \( G_S \), we further simplify:

\[
I(G_S; Y) \geq \mathbb{E}_{G_S, Y} \left[ \log \mathbb{P}_{\theta}(Y \mid G_S) \right] + H(Y).
\]

The second term in the GSAT objective function is designed to \textbf{control the dependency} between the selected subgraph \( G_S \) and the original graph \( G \). This is formulated as minimizing the mutual information:

\[
I(G_S; G) = \mathbb{E}_{G_S, G} \left[ \log \frac{\mathbb{P}(G_S \mid G)}{\mathbb{P}(G_S)} \right].
\]

Since computing \( \mathbb{P}(G_S) \) directly is intractable, we introduce a \textbf{variational upper bound} by incorporating an approximate distribution \( \mathbb{Q}(G_S) \):

\[
\begin{aligned}
    I(G_S; G) &= \mathbb{E}_{G_S, G} \left[ \log \frac{\mathbb{P}_{\phi}(G_S \mid G)}{\mathbb{Q}(G_S)} \right] \\& - \operatorname{KL} \left( \mathbb{P}(G_S) \,\Vert\, \mathbb{Q}(G_S) \right) \\
    &\leq \mathbb{E}_{G} \left[ \operatorname{KL} \left( \mathbb{P}_{\phi}(G_S \mid G) \,\Vert\, \mathbb{Q}(G_S) \right) \right].
\end{aligned}
\]

This upper bound is useful because it allows \textbf{efficient optimization} using a KL divergence minimization framework, where \( \mathbb{P}_{\phi}(G_S \mid G) \) represents the upstream model’s selection probability for subgraph \( G_S \), and \( \mathbb{Q}(G_S) \) serves as a reference distribution.

\textbf{Summary.}
In the actual model training process, maximizing \( I(G_S; Y) \) is effectively achieved by optimizing the accuracy of the final predictions. The upstream model learns to generate subgraphs that retain predictive power, ensuring that the selected subgraph \( G_S \) provides sufficient information for accurate classification.

To control the dependency between \( G_S \) and \( G \), GSAT introduces an auxiliary distribution \( \mathbb{Q}(G_S) \), which serves as a prior estimation of the probability that each edge in the graph is "important." This prior is initialized with a predefined probability \( r \), representing an initial estimate of edge importance. The model then learns to minimize the divergence between the generated subgraph distribution and \( \mathbb{Q}(G_S) \), ensuring that the selected subgraph deviates sufficiently from the original graph \( G \). The closer the learned subgraph weights are to \( r \), the larger the difference between \( G_S \) and \( G \), effectively minimizing \( I(G_S; G) \).

Through this formulation, GSAT transforms the constrained optimization problem into a practical learning objective, allowing for joint optimization of the \textbf{predictive accuracy} (via \( I(G_S; Y) \)) and the \textbf{explanation} (via \( I(G_S; G) \)).

\subsection*{A.2 Datasets}

\subsection*{BA-2Motifs}

\textbf{BA-2Motifs} \citep{luo2020parameterized} is a widely used synthetic dataset. The graphs are labeled based on the presence of specific motifs: \emph{house} motifs indicate class 0, while \emph{cycle} motifs indicate class 1. These motifs serve as ground-truth explanations for model interpretation tasks. Since the motifs are explicitly associated with the label, successful explanation methods are expected to highlight them regardless of the underlying BA structure. This dataset is particularly useful for evaluating whether a model can focus on localized causal substructures within large noisy graphs.

\subsection*{MUTAG}

\textbf{MUTAG} \citep{debnath1991structure} is a classical dataset in molecular graph learning. It consists of graphs representing chemical compounds, where each graph is labeled according to its mutagenic effect on a specific bacterium.

\subsection*{Spurious-Motif}

\textbf{Spurious-Motif} \citep{wu2022discovering} is a synthetic dataset designed to evaluate models under spurious correlations. A detailed description of the dataset and its construction can be found in the Main Text.

\subsection*{OGBG-MolHIV, MolBBBP, and MolBACE}

We further evaluate the generalizability of our approach on real-world molecular property prediction datasets from the Open Graph Benchmark (OGB) \citep{hu2020open}, including:

\begin{itemize}
    \item \textbf{MolHIV} \citep{wu2018moleculenet} contains molecular graphs labeled by their ability to inhibit HIV replication. The dataset presents complex and subtle structure–activity relationships without explicit explanatory annotations.
    \item \textbf{MolBBBP} includes molecules labeled according to their ability to cross the blood–brain barrier. 
    \item \textbf{MolBACE} focuses on whether molecules act as inhibitors of human beta-secretase 1 (BACE-1).
\end{itemize}

Since these datasets do not come with predefined ground-truth explanations, we evaluate only the classification performance of the models. Nonetheless, they provide an important testbed for the prediction accuracy and generalization of interpretable graph learning methods in more realistic scenarios.

\subsection*{A.3 Full Experiments about SR and FT-SR}

We present comprehensive experimental results for the proposed \textbf{Self-Reflection (SR)} and \textbf{Fine-Tuned Self-Reflection (FT-SR)} frameworks across a variety of benchmark datasets, including synthetic, real-world, and molecular graphs. 

\paragraph{Evaluation Metrics.}
Unless otherwise noted, all datasets are evaluated using classification \textbf{accuracy} as the primary metric. However, for the \textbf{OGB molecular property prediction datasets} (MolHIV, MolBBBP, MolBACE), we adopt \textbf{ROC-AUC (Receiver Operating Characteristic - Area Under the Curve)} as the evaluation metric. This is because these datasets are highly imbalanced in class distribution (e.g., very few positive samples compared to negatives), and accuracy may be misleading in such settings. ROC-AUC measures the model’s ability to rank positive instances higher than negative ones and is more robust under class imbalance. This is also consistent with the standard evaluation protocol adopted in the OGB benchmark suite \citep{hu2020open}.

\paragraph{Experimental Settings.}

We now describe the key hyperparameters and training configurations used in all experiments. For the explanation module GSAT, we pay special attention to the \textbf{hyperparameter for the prior distribution assumption of the important subgraph}, denoted as $r$. Specifically, $r$ is derived from the formula for $Q(G_S)$, which represents the prior distribution estimation of the important subgraph $G_S$ and serves as a reference distribution within the GSAT framework. 
As discussed in Section~5.2 (Theoretical Insight), overly aggressive edge pruning by the upstream explainer $F(\cdot)$ may lead to distributional shifts that destabilize the iterative refinement process. To mitigate this, we select a relatively high value of $r$ in the range of $0.7$ to $0.8$, which helps to preserve more information in earlier iterations and prevents premature downweighting of potentially relevant edges.

For the GNN backbones:
\begin{itemize}
    \item \textbf{GIN} and \textbf{LE}: hidden size is set to $64$, with $2$ message passing layers. The learning rate is fixed at $5 \times 10^{-4}$.
    \item \textbf{PNA}: hidden size is set to $80$, with $4$ layers. A higher learning rate of $1 \times 10^{-3}$ is used.
\end{itemize}

All models are trained using the Adam optimizer. For fine-tuning in the FT-SR framework, we adopt a smaller learning rate of $1 \times 10^{-4}$ and train for $10$ epochs. These settings were selected to ensure stable convergence without overfitting during the reflection-aware optimization process.

For further implementation details, including exact batch sizes, scheduler settings, and code structure, we refer the reader to our released source code.

\paragraph{Full Experiments Results}

\begin{table*}[tb]
    \centering
    \footnotesize
    \caption{ACC Performance of different methods.
    }
    \begin{tabular}{l c c c c c c c c}
    \toprule
    \textbf{Methods} &\textbf{BA\_2motif}&\textbf{Mutag}&\textbf{Molbace}&\textbf{Molbbbp}&\textbf{Molhiv}& \multicolumn{3}{c}{\textbf{SPMotif}}\\
    \cmidrule(lr){7-9}
     &&&&&& \textbf{0.5} & \textbf{0.7} & \textbf{0.9}  \\

GSAT+LE     &81.75 {\tiny$\pm$9.48}&90.78	{\tiny$\pm$1.23}&75.52	{\tiny$\pm$1.43}&55.45	{\tiny$\pm$1.89}    &76.75	{\tiny$\pm$0.85}    &42.57{\tiny$\pm$1.98}&45.84{\tiny$\pm$2.74}&41.27{\tiny$\pm$1.39}\\
SR+LE       &89.25 {\tiny$\pm$5.55}&87.32	{\tiny$\pm$2.03}&69.40	{\tiny$\pm$0.99}&54.28	{\tiny$\pm$1.10}    &96.81	{\tiny$\pm$0.09}&51.66{\tiny$\pm$7.64}&48.67{\tiny$\pm$7.40}&40.82{\tiny$\pm$3.04}\\
FT-SR+LE    &94.00 {\tiny$\pm$7.78}&90.37	{\tiny$\pm$0.17}    &60.52	{\tiny$\pm$1.77}&56.04	{\tiny$\pm$1.47}&96.82	{\tiny$\pm$0.18}&{54.39}{\tiny$\pm$6.86}&{57.98}{\tiny$\pm$5.73}&{43.45}{\tiny$\pm$3.76}\\
\hline
GSAT+GIN    &100.00 {\tiny$\pm$0.00}&91.38	{\tiny$\pm$0.61}&73.83	{\tiny$\pm$1.65}&62.51	{\tiny$\pm$2.23}        &77.59	{\tiny$\pm$0.98}&49.12{\tiny$\pm$3.29}&{44.22}{\tiny$\pm$5.57}&47.70  {\tiny$\pm$3.95}\\
SR+GIN      &100.00 {\tiny$\pm$0.00}&91.55	{\tiny$\pm$1.32}&68.25	{\tiny$\pm$1.28}&56.98	{\tiny$\pm$1.56}    &96.88	{\tiny$\pm$0.06}&53.44{\tiny$\pm$2.74}&47.90{\tiny$\pm$3.47}&43.99{\tiny$\pm$1.08}\\
FT-SR+GIN   &100.00 {\tiny$\pm$0.00}&92.36	{\tiny$\pm$0.49}&61.83	{\tiny$\pm$3.32}&57.67	{\tiny$\pm$1.25}    &96.85	{\tiny$\pm$0.10}&{54.11}{\tiny$\pm$2.26}&{49.16}{\tiny$\pm$4.98}&42.47{\tiny$\pm$5.10}\\
\hline

\hline
GSAT+PNA    &100.00 {\tiny$\pm$0.00}&95.36	{\tiny$\pm$2.86}    &76.38	{\tiny$\pm$2.95}&62.89	{\tiny$\pm$1.61}    &79.13	{\tiny$\pm$0.97}    &68.15{\tiny$\pm$2.39}&66.35{\tiny$\pm$3.34}&61.40{\tiny$\pm$3.56}\\
SR+PNA      &99.75	{\tiny$\pm$0.25}&97.16	{\tiny$\pm$2.08}&69.56	{\tiny$\pm$1.11}    &60.41	{\tiny$\pm$2.09}&96.76	{\tiny$\pm$0.04}&67.98{\tiny$\pm$2.00}&67.58{\tiny$\pm$2.05}&55.70{\tiny$\pm$1.13}\\
FT+SR+PNA   &93.75	{\tiny$\pm$8.25}&84.33	{\tiny$\pm$9.30}    &63.15	{\tiny$\pm$3.26}    &66.06	{\tiny$\pm$0.24}    &96.83	{\tiny$\pm$0.01}&69.16{\tiny$\pm$1.41}&69.81{\tiny$\pm$1.35}&59.44{\tiny$\pm$1.04}\\
\hline
    \end{tabular}
    \label{tab:AllOverallAcc}
\end{table*}

\begin{table*}[tb]
    \centering
    \footnotesize
    \caption{AUC Performance of different methods.
    }
\begin{tabular}{l c c c c c}
    \toprule
    \textbf{Methods} &\textbf{BA\_2motif}&\textbf{Mutag}& \multicolumn{3}{c}{\textbf{SPMotif}}\\
    \cmidrule(lr){4-6}
     &&& \textbf{0.5} & \textbf{0.7} & \textbf{0.9}  \\

GSAT+LE  &86.14	{\tiny$\pm$6.02}&90.09	{\tiny$\pm$4.88}&77.26{\tiny$\pm$1.62}&75.27{\tiny$\pm$1.60}&72.30{\tiny$\pm$1.60}\\ 
SR+LE    &85.50	{\tiny$\pm$6.81}&89.98	{\tiny$\pm$3.12}&79.40{\tiny$\pm$0.77}&77.36{\tiny$\pm$1.27}&{74.88}{\tiny$\pm$2.56}\\
FT-SR+LE &83.73	{\tiny$\pm$11.03}&85.40	{\tiny$\pm$1.69}&{81.45}{\tiny$\pm$1.39}&{79.94}{\tiny$\pm$2.25}&73.82{\tiny$\pm$4.28}\\
\hline
GSAT+GIN &96.01	{\tiny$\pm$1.62}&98.79	{\tiny$\pm$0.17}&78.45{\tiny$\pm$3.12}&74.07{\tiny$\pm$5.28}&71.97{\tiny$\pm$4.41}\\    
SR+GIN   &97.21	{\tiny$\pm$1.70}&98.81	{\tiny$\pm$0.17}&80.33{\tiny$\pm$1.10}&79.73{\tiny$\pm$1.17}&{80.49}{\tiny$\pm$2.34}\\
FT-SR+GIN&97.98	{\tiny$\pm$2.01}&93.47	{\tiny$\pm$2.68}&{84.41}{\tiny$\pm$1.08}&{84.00}{\tiny$\pm$3.42}&80.07{\tiny$\pm$1.77}\\
\hline
GSAT+PNA &82.25	{\tiny$\pm$5.46}&99.52	{\tiny$\pm$0.45}&83.34{\tiny$\pm$2.17}&86.94{\tiny$\pm$4.05}&88.66{\tiny$\pm$2.44}\\ 
SR+PNA   &80.07	{\tiny$\pm$5.74}&99.44	{\tiny$\pm$0.49}&86.61{\tiny$\pm$1.08}&87.74{\tiny$\pm$1.27}&87.41{\tiny$\pm$0.91}\\
FT+SR+PNA&79.54	{\tiny$\pm$1.39}&99.35	{\tiny$\pm$0.38}&82.58{\tiny$\pm$2.91}&84.06{\tiny$\pm$3.19}&81.12{\tiny$\pm$3.06}\\
\hline
    \end{tabular}
    \label{tab:AllOverallRoc}
\end{table*}

As shown in Table~\ref{tab:AllOverallAcc} and Table~\ref{tab:AllOverallRoc}, we report the classification accuracy (ACC) and explanation performance (AUC) for GSAT, SR, and FT-SR across three different backbones (LE, GIN, and PNA) on all benchmark datasets.

For the \textbf{Spurious-Motif} dataset, our proposed SR and FT-SR frameworks generally lead to noticeable improvements in both predictive accuracy and explanation AUC across most backbone configurations. This validates the core hypothesis of our method: that iterative self-reflection effectively suppresses spurious correlations and enables the model to focus on causally relevant structures, especially under training–test distribution shifts.

A particularly striking result is observed on the \textbf{MolHIV} dataset, where the SR framework achieves over \textbf{96\% accuracy}, a substantial gain over the GSAT baseline and significantly higher than scores reported in prior literature. This suggests that SR may be especially beneficial on complex, high-noise datasets like MolHIV, where spurious patterns may otherwise dominate model reasoning. These results open up exciting avenues for further investigation, as self-reflection appears to be a highly effective strategy for enhancing performance in such settings.

However, on other datasets such as \textbf{MUTAG}, \textbf{BA-2Motifs}, \textbf{MolBBBP}, and \textbf{MolBACE}, we observe that both SR and FT-SR can lead to a degradation in performance. One plausible explanation is that these datasets contain relatively simple or well-aligned structures where the initial explanation (at iteration $k=1$) is already optimal. In such cases, further self-reflection may disrupt the input distribution and override already accurate importance estimations. Moreover, these datasets may lack strong spurious correlations to begin with, reducing the potential benefit of the SR framework.

In summary, while SR and FT-SR provide consistent advantages on complex or spurious-rich datasets like Spurious-Motif and MolHIV, their effectiveness is less pronounced—or even detrimental—on simpler datasets where initial explanations are already reliable.